\documentclass[]{fairmeta}

\title{\vspace{-1pt}EfficientFlow: Efficient Equivariant Flow Policy Learning for Embodied AI}

\author[*]{Jianlei Chang}
\author[*]{Ruofeng Mei}
\author{Wei Ke}
\author[\dagger]{Xiangyu Xu}

\affiliation{Xi'an Jiaotong University}

\contribution[*]{Equal Contribution}
\contribution[\dagger]{Corresponding author}

\abstract{
\vspace{-5pt}
Generative modeling has recently shown remarkable promise for visuomotor policy learning, enabling flexible and expressive control across diverse embodied AI tasks. 
However, existing generative policies often struggle with \emph{data inefficiency}, requiring large-scale demonstrations, and \emph{sampling inefficiency}, incurring slow action generation during inference.
We introduce EfficientFlow, a unified framework for efficient embodied AI with flow-based policy learning. 
To enhance data efficiency, we bring equivariance into flow matching. We theoretically prove that when using an isotropic Gaussian prior and an equivariant velocity prediction network, the resulting action distribution remains equivariant, leading to improved generalization and substantially reduced data demands.
To accelerate sampling, we propose a novel acceleration regularization strategy. As direct computation of acceleration is intractable for marginal flow trajectories, we derive a novel surrogate loss that enables stable and scalable training using only conditional trajectories.
Across a wide range of robotic manipulation benchmarks, the proposed algorithm achieves competitive or superior performance under limited data while offering dramatically faster inference. These results highlight EfficientFlow as a powerful and efficient paradigm for high-performance embodied AI.
\vspace{-10pt}
}

\metadata[Code]{\url{https://github.com/chang-jl/EfficientFlow}}
\metadata[Website]{\url{https://efficientflow.github.io}}
\usepackage{hyperref}
\usepackage{url}
\usepackage[T1]{fontenc}    %
\usepackage{url}            %
\usepackage{booktabs}       %
\usepackage{amsfonts}       %
\usepackage{nicefrac}       %
\usepackage{microtype}      %
\usepackage{epsfig, xspace, enumitem}
\usepackage{graphicx, amsmath, amssymb, caption, subcaption, multirow, overpic, textpos, pifont, adjustbox}

\usepackage{xcolor}
\usepackage[utf8x]{inputenc}
\makeatletter
\@namedef{ver@everyshi.sty}{}
\makeatother
\usepackage{pgf, tikz, pgfplots}
\usepackage{makecell}
\usepackage{pgf-pie}
\usepackage{float}
\usepackage{wrapfig}
\usepackage{colortbl}

\usepackage{newfloat}
\usepackage{listings}
\usepackage{booktabs}
\usepackage{siunitx}
\usepackage{multirow}
\usepackage{amsmath}
\usepackage{amssymb}
\usepackage{amsthm}
\newtheorem{theorem}{Theorem}
\newtheorem{lemma}{Lemma}
\theoremstyle{definition}

\theoremstyle{remark}

\definecolor{citecolor}{HTML}{0071BC}
\definecolor{linkcolor}{HTML}{ED1C24}
\definecolor{acceptcolor}{HTML}{74C219}
\definecolor{rejectcolor}{HTML}{DE1616}
\definecolor{qcolor}{HTML}{536872}
\definecolor{demphcolor}{RGB}{100,100,100}
\definecolor{brightlavender}{rgb}{0.75, 0.58, 0.89}
\definecolor{palered}{rgb}{1.00, 0.70, 0.70}
\definecolor{palegreen}{rgb}{0.73, 0.96, 0.67}
\definecolor{paleblue}{rgb}{0.69, 0.84, 1.00}
\definecolor{paleorange}{rgb}{1.00, 0.86, 0.73}
\definecolor{palepurple}{rgb}{0.92, 0.85, 1.00}
\definecolor{paleyellow}{rgb}{1.00, 1.00, 0.50}

\newlength\savewidth

\renewcommand{\paragraph}[1]{\vspace{1.25mm}\noindent\textbf{#1}}
\newcolumntype{L}[1]{>{\raggedright\let\newline\\\arraybackslash\hspace{0pt}}m{#1}}
\newcommand{\app}{\raise.17ex\hbox{$\scriptstyle\sim$}}

\definecolor{lightgray}{rgb}{0.95, 0.95, 0.95}

\definecolor{baselinecolor}{gray}{.9}

\setlist[enumerate]{itemsep=-0.5mm,partopsep=0pt}

\renewcommand{\paragraph}[1]{\vspace{1.25mm}\noindent\textbf{#1}}

\newcolumntype{x}[1]{>{\centering\arraybackslash}p{#1pt}}
\newcolumntype{y}[1]{>{\raggedright\arraybackslash}p{#1pt}}
\newcolumntype{z}[1]{>{\raggedleft\arraybackslash}p{#1pt}}
\pgfplotsset{compat=1.18}
\begin{document}

\maketitle

\section{Introduction}

Learning robotic policies from data using generative models has emerged as a powerful and flexible paradigm in embodied AI, particularly with the recent success of diffusion-based approaches~\citep{chi2023diffusion, Ze2024DP3}. 
These models have demonstrated strong performance in visuomotor control by learning complex action distributions conditioned on high-dimensional observations. 
However, two key limitations remain: low data efficiency, requiring large amounts of training data, and low sampling efficiency, incurring high computational cost at inference due to the iterative sampling process.

Recent works have sought to address the data efficiency issue by incorporating equivariance into diffusion models for policy learning~\citep{wang2024equivariant}. 
By leveraging the inherent symmetries of the environment (e.g., 2D rotation), these methods introduce strong inductive biases that enable policies to generalize across symmetric configurations. 
Nevertheless, as they are still built upon diffusion models, which typically require hundreds of iterative denoising steps to generate a single action~\citep{sohl2015deep,ho2020denoising}, they remain impractical for real-time robotic control.
To overcome this limitation, we turn to Flow Matching~\citep{lipman2022flow}, a recent class of generative models that learns a continuous trajectory from a simple prior distribution to the data distribution using an ordinary differential equation (ODE) defined by a velocity field. 
Compared to diffusion models, flow-based approaches offer better numerical stability and faster inference, making them highly appealing for efficient embodied AI.

We present EfficientFlow, a new policy learning framework that unifies equivariant learning and flow-based generative modeling.
We first investigate how to incorporate equivariance into flow-based policy models and theoretically show that, under an isotropic Gaussian prior and an equivariant velocity field network, the conditional action distribution induced by flow matching remains equivariant with respect to input observation transformations (see Figure~\ref{fig1}(a)). 
This property allows policies to generalize across symmetric configurations of the environment without additional supervision or data augmentation.

To further improve the action sampling efficiency, we introduce a regularization technique that penalizes the acceleration of the generation flow trajectory, i.e., the second-order temporal derivative, which encourages a smoother and more stable action sampling process.

However, computing acceleration requires consecutive points along the marginal flow trajectories, which are unavailable in the standard flow matching framework.
To address this challenge, we propose a novel surrogate objective called Flow Acceleration Upper Bound (FABO). 
FABO provides a practical and effective approximation of the acceleration penalty using only conditional flow trajectories available during training, enabling much faster flow policies with lower computational costs.

The proposed EfficientFlow combines the best of both worlds: it achieves fast inference speed thanks to the flow-based architecture and smoothed sampling trajectory, and maintains high performance by leveraging equivariance. 
As illustrated in Figure~\ref{fig1}(b), EfficientFlow compares favorably against existing methods in both inference speed and task success rates.

Our primary contributions are as follows:
\begin{itemize}
\item We formulate a flow-based policy learning framework, EfficientFlow, that achieves equivariance to geometric transformations, allowing the model to generalize across symmetric states and significantly improve data efficiency. We provide a theoretical analysis showing that equivariance is preserved in the flow framework when using an isotropic prior and an equivariant velocity field conditioned on visual observations.

\item To promote sampling speed, we propose a second-order regularization objective that penalizes flow acceleration. Since direct acceleration computation requires access to neighboring marginal samples that are unavailable, we introduce a novel surrogate loss called FABO, enabling effective training.

\item We provide comprehensive evaluations of EfficientFlow on 12 robotic manipulation tasks in the \mbox{MimicGen}~\citep{mandlekar2023mimicgen} benchmark, showing that EfficientFlow achieves favorable success rates with high inference speeds (19.9 to 56.1 times faster than EquiDiff~\citep{wang2024equivariant}).
\end{itemize}

\begin{figure*}[t]
\centering
\includegraphics[width=1\linewidth]{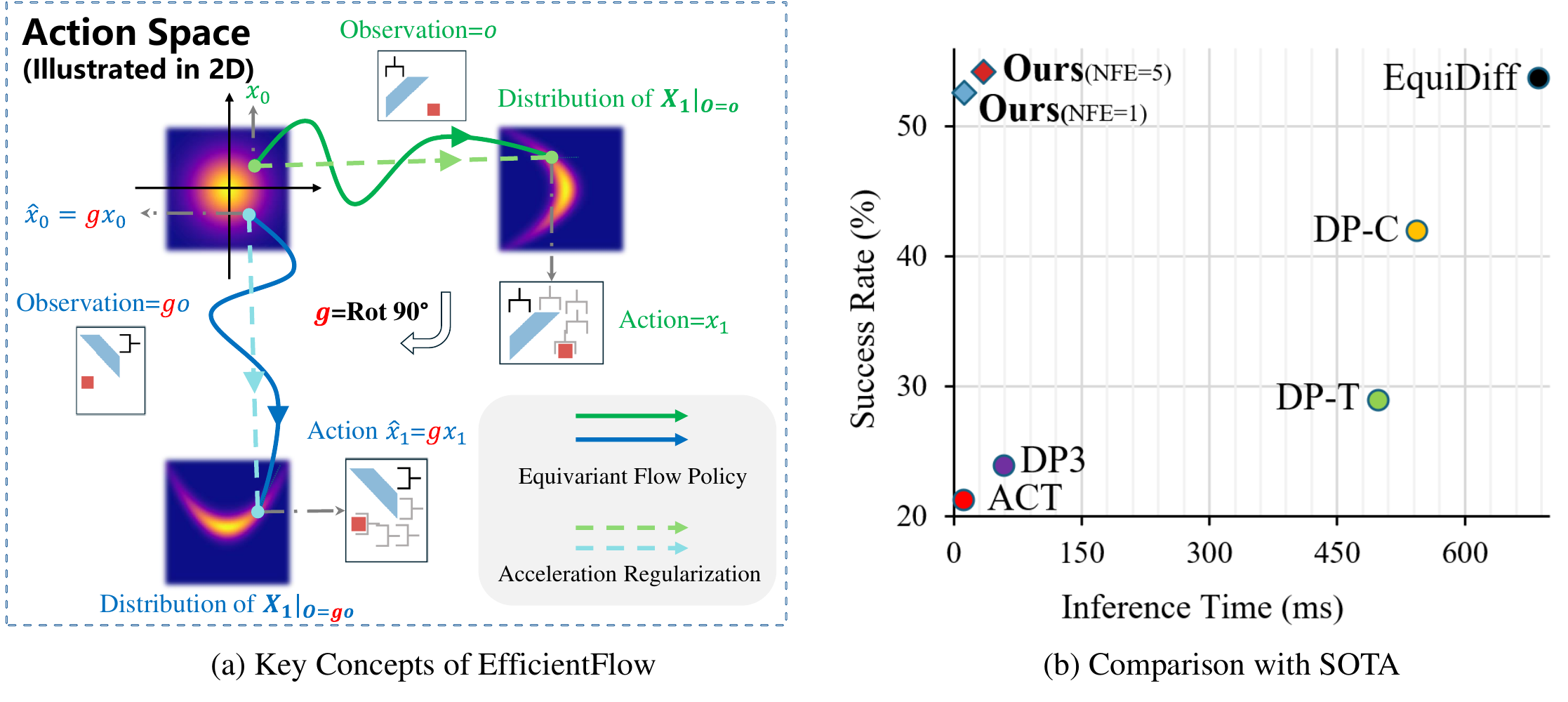} 
\caption{We propose EfficientFlow (a) to effectively combine equivariance with Flow Policy and introduce an acceleration regularization to achieve high-quality, fast action generation. As shown in (b), EfficientFlow compares favorably against baseline policy learning approaches in both success rate and inference speed. Results are from MimicGen with 100 training demonstrations.}
\label{fig1}
\end{figure*}

\section{Related Work}

\subsection{Equivariance in Robot Manipulation}

Applying equivariance to robot manipulation is a highly promising research direction, and multiple studies have demonstrated that it can significantly enhance the data efficiency of robot policy learning~\citep{wangmathrm,jia2023seil,wang2022equivariant,simeonov2023se,pan2023tax,huang2023edge,liu2023continual,kim2023se,nguyen2023equivariant,yang2024equibot}. Early work used SE(3) open-loop or SE(2) closed-loop for control, validated the effectiveness of equivariant models in on-robot learning~\citep{zhu2022sample,wangrobot,zhu2023robot}, and achieved pick-and-place tasks based on few-shot demonstrations~\citep{huang2022equivariant,simeonov2022neural,ryuequivariant,huang2024fourier}.
Building on this foundation, EquiDiff~\citep{wang2024equivariant} extend the research to the SE(3) closed-loop action space, substantially improving the efficiency of imitation learning by integrating symmetry with diffusion policies. However, the DDPM architecture employed by EquiDiff requires a multi-step denoising process, resulting in slow inference speeds. In contrast, the EfficientFlow model marks a significant breakthrough in inference efficiency, attaining higher success rates than EquiDiff with only a minimal number of inference steps.

\subsection{Flow Policy}

Flow Matching~\citep{lipman2022flow} represents a novel class of generative models grounded in optimal transport theory. Its objective is to learn a vector field of a probability path, which is more efficient than diffusion paths, offering faster training and sampling, alongside better generalization capabilities. Compared to diffusion models, Flow Matching significantly reduces the number of inference steps, a critical factor for real-world robotic operations, thereby substantially broadening the applicability of such models. The work Flow Policy~\citep{zhang2025flowpolicy} introduced conditional Consistent Flow Matching~\citep{yang2024consistency} to robotic manipulation. Conditioned on observed 3D point clouds, Flow Policy utilizes Consistency Flow Matching to directly define straight-line flows from different temporal states to the same action space, concurrently constraining their velocity values. It approximates trajectories from noise to robot actions by normalizing the self-consistency of the velocity field within the action space, thereby enhancing inference efficiency. MP1~\citep{sheng2025mp1} leverages Mean Flow~\citep{geng2025mean} to shrink policy learning to a single state-action step, while a lightweight Dispersive Loss repels state embeddings. This combination steadies the flow field and delivers millisecond inference that outpaces DP3 and Flow Policy. Currently, many VLA (Vision-Language-Action) models~\citep{black2024pi_0,gao2025vita,bjorck2025gr00t,reussflower} are utilizing flow matching policies and have achieved good results.

\section{Method}
\subsection{Preliminaries}
\subsubsection{Flow Matching}
The core idea of Flow Matching~\citep{lipman2022flow} is to learn the vector field of an ODE that smoothly transforms samples $x_0$ from a simple prior distribution $p_0$ (e.g., Gaussian noise) to samples $x_1$ from a target data distribution $p_1$. 

Specifically, let \(\{p_t\}_{t \in [0,1]}\) be a time-evolving family of probability distributions satisfying the boundary conditions \(p_{t=0}=p_0\) and \(p_{t=1}=p_1\). This path induces an underlying ground-truth instantaneous velocity field \(u^{\text{gt}}(t, x)\).
Flow Matching aims to learn a vector field $u_\theta(t,x)$ parameterized by $\theta$, such that trajectories $x_t$ defined by the following ODE:
\begin{equation}\label{eq:ode}
\begin{cases}
\frac{dx_t}{dt} = u_\theta(t,x_t) \\
x_0\sim p_0
\end{cases}
\end{equation}
can effectively transport the prior distribution \(p_0\) to the target distribution \(p_1\). 
Ideally, the learned vector field $u_\theta(t,x)$ should approximate the true vector field $u^{\text{gt}}(t,x)$. 
Thus, a natural learning objective for Flow Matching is:
\begin{equation}\label{eq:fm}
L_{\text{FM}} = \mathbb{E}_{t,x_t} \left[ \left\| u_\theta(t,x_t) - u^{\text{gt}}(t,x_t) \right\|_2^2 \right],
\end{equation}
where $x_t \sim p_t$. 

As $u^{\text{gt}}(t,x_t)$ is generally intractable in practice, Conditional Flow Matching (CFM)~\citep{lipman2022flow} proposes to learn $u_\theta(t,x_t)$ by regressing against a conditional vector field $u(t,x_t|x_1)$, using samples from a conditional probability path $p_t(x|x_1)$.
The corresponding objective is:
\begin{equation}\label{eq:CFM}
L_{\text{CFM}} = \mathbb{E}_{t,x_1,x_t} \left[ \left\| u_\theta(t,x_t) - u(t,x_t|x_1) \right\|_2^2 \right],
\end{equation}
where $t\sim U(0,1),x_1 \sim p_1(x)$, and $x_t\sim p_t(x|x_1)$.

\subsubsection{Equivariance}
Equivariance is a desirable property in many learning systems, especially when modeling structured data influenced by known symmetries~\citep{cesa2022program}. 
A function \( f \) is said to be {equivariant} with respect to a transformation group \( G \) if it commutes with the actions of the group. 
Formally, this is expressed as:
\begin{equation}\label{eq:equiv}
    f(\rho_x(g)x) = \rho_y(g)f(x), \quad \forall g \in G,
\end{equation}
where \( \rho_x \) and \( \rho_y \) denote group representations that describe how the group acts on the input space and output space, respectively. This equation ensures that applying a group transformation to the input and then evaluating the function (Eq.~\ref{eq:equiv} left) yields the same result as first applying the function and then transforming the output (Eq.~\ref{eq:equiv} right).

In this work, we focus on learning equivariant policies for robot arm control, where the input \( x \) represents the robot arm action in the task space. 
Since robotic manipulation tasks often exhibit rotational symmetry~\citep{wang2024equivariant,wangmathrm}, we study the action of the rotation group \( \mathrm{SO}(2) \) and its finite cyclic subgroup \( C_u \subset \text{SO}(2) \), which models discrete rotational symmetries (e.g., rotations by \( \frac{2\pi}{u} \) radians).

We consider the following standard representations:
1) the trivial representation \( \rho_0 \), which maps every group element \( g \in G \) to the identity transformation. This is typically used when the function output should remain invariant under the group action.
2) the {standard irreducible representation} \( \rho_1 \), which describes the canonical action of \( \mathrm{SO}(2) \) or \( C_u \) on the 2D plane, defined as \(
\rho_1(g) = 
\begin{bmatrix}
\cos(g) & -\sin(g) \\
\sin(g) & \cos(g)
\end{bmatrix},
\)
where we slightly abuse the notation to use $g$ to denote both a group element and its corresponding rotation angle.

By designing policy networks that are equivariant under these group actions, we aim to incorporate inductive biases that reflect the underlying symmetries of the robot's action space. This not only improves sample efficiency but also enhances generalization across real task configurations.

\begin{figure*}[ht]
  \centering
  \includegraphics[width=0.95\textwidth]{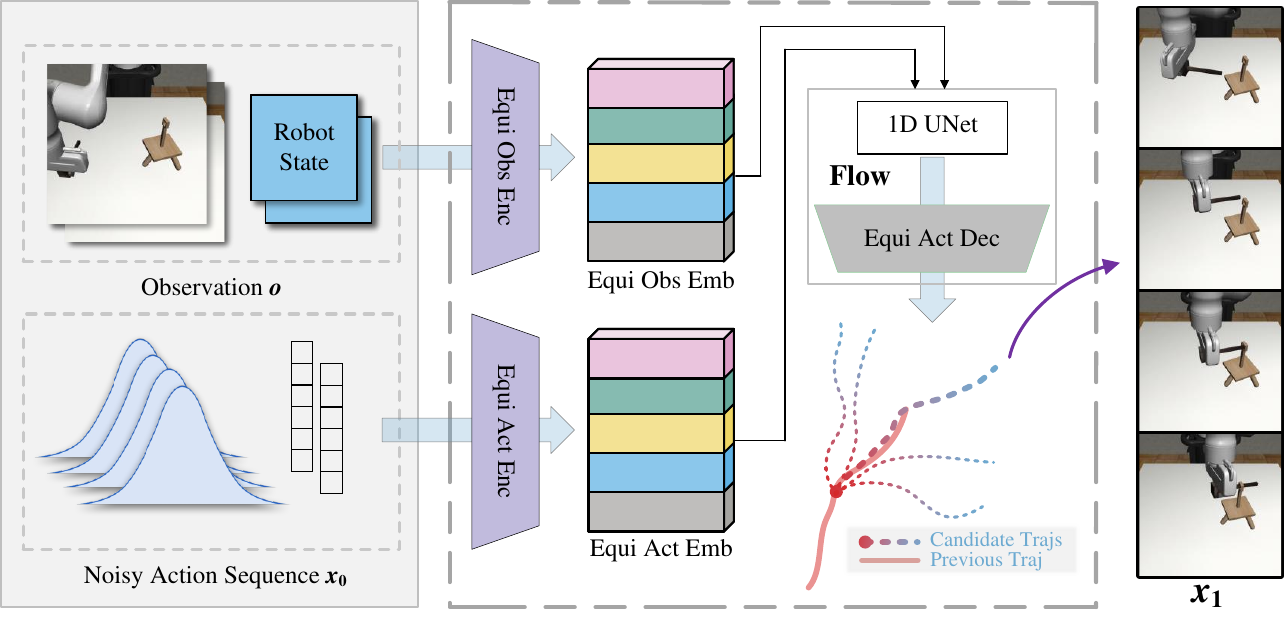}
  \caption{Overview of EfficientFlow. At each decision step, the policy utilizes the most recent two observation steps $\textbf{o}$ as input. This information is processed by the equivariant Flow Matching network to generate five candidate action trajectories. The trajectory that exhibits the minimum Euclidean distance to the previously predicted trajectory is then selected for execution, ensuring a smooth and coherent action sequence.}
  \label{image-main}
\end{figure*}

\subsection{Equivariant Flow Policy}
Generative models for policy learning have received significant attention in recent years. Given an observation \( o \), such models can predict a conditional distribution \( p_{X_1 | O = o} \), and generate robot actions by sampling $x \sim X_1 \big|_{O=o}$, where \( X_1 \) represents the random variable for the action to be executed by the robot arm under the condition that $O=o$.

Both $o$ and $x$ can span multiple time steps: $o=[o^{\tau-(m-1)}, \cdots, o^{\tau-1}, o^\tau]$, $x=[x^\tau, x^{\tau+1}, \cdots, x^{\tau+(n-1)}]$, where $m$ is the number of historical observations, and $n$ is the number of future action steps.
The observation $o^\tau$ includes both the image and the robot state at robot time $\tau$.

A desirable property for such models is {equivariance}: when the input \( o \) is transformed by an element \( g \in G \) of a symmetry group (e.g., a rotation), the conditional distribution of the output action should transform accordingly. In other words, symmetry in the observation space should induce symmetry in the action space:
\begin{align}\label{eq:desired property}
    X_1 \big|_{O = go} \stackrel{d}{=} g \left( X_1 \big|_{O = o} \right),
\end{align}
where $\stackrel{d}{=}$ denotes that the two random variables have the same distribution. We leave the group representation $\rho(g)$ implicit here and directly use $g$ for brevity.

\subsubsection{How to Make Flow Policy Equivariant?}
The main contribution of this work is to demonstrate that the desired property in Eq.~\ref{eq:desired property} can be achieved within the Flow Matching framework by: 
\begin{enumerate}
    \item using an isotropic distribution for \( p_0 \) in Eq.~\ref{eq:ode}, e.g., Gaussian noise \( X_0 \sim \mathcal{N}(0, I) \);
    \item using an equivariant network \( u_\theta \) for the velocity field such that:
\begin{align}\label{eq:uequiv}
    u_\theta(t, g x | g o) = g \left( u_\theta(t, x | o) \right), \quad \forall g \in G.
\end{align}
\end{enumerate}
Importantly, we do \emph{not} impose the strong assumption that the expert policy in the training data be equivariant, which is in sharp contrast with \citep{wang2024equivariant}.

\begin{theorem}\label{th:equi}
Let \( G \) be a transformation group acting on both the observation space and the action space. Suppose the initial distribution \( p_0 \) is isotropic, i.e., \( p_0(g x) = p_0(x) \) for all \( g \in G \), and the velocity network \( u_\theta(t, x | o) \) is equivariant as in Eq.~\ref{eq:uequiv}.
Then the induced conditional distribution at time \( t \), given by the flow ODE Eq.~\ref{eq:ode}, satisfies
\begin{align}\label{eq:theorem}
    X_t \big|_{O = go} \stackrel{d}{=} g \left( X_t \big|_{O = o} \right), \quad t\in [0,1]
\end{align}
i.e., the output distribution is equivariant under the group action.
\end{theorem}

The special case $t=1$ of Eq.~\ref{eq:theorem} gives us the desired property in Eq.~\ref{eq:desired property}.
An intuitive visualization of this result is provided in Figure~\ref{fig1}(a).
From a discrete-time perspective, consider starting from a randomly sampled initial action $x_0 \sim p_0$. After a small time step $\Delta t$, the action evolves under the velocity field to reach $x_{\Delta t} = x_0 + \Delta t \cdot u_\theta(0, x_0 | o) $.
This corresponds to the green curve in Figure~\ref{fig1}(a).

Now, consider a rotated scenario where the initial action is transformed to $\hat{x}_0 = g x_0$, and the observation is rotated accordingly to $go$. The updated action becomes
\begin{align}
\hat{x}_{\Delta t} 
=&\ \hat{x}_0 + \Delta t \cdot u_\theta(0, \hat{x}_0 | go) 
&& \text{(one-step update)} \nonumber\\
=&\ g x_0 + \Delta t \cdot u_\theta(0, g x_0 | go) 
&& \text{(since $\hat{x}_0 = g x_0$)} \nonumber\\
=&\ g x_0 + \Delta t \cdot g u_\theta(0, x_0 | o) 
&& \text{(Eq.~\ref{eq:uequiv})} \nonumber\\
=&\ g x_{\Delta t}.  \nonumber
\end{align}
This corresponds to the blue curve in Figure~\ref{fig1}(a), showing that the evolution of the rotated action $\hat{x}_t$ under the rotated observation $go$ aligns with the rotated evolution of the original action $x_t$.

By repeating this process over the entire flow trajectory, we conclude that $\hat{x}_1 = g x_1$ (see Figure~\ref{fig1}(a)).
Since $p_0$ is isotropic, $x_0$ and $\hat{x}_0$ have the same probability density. 
Given that the flow deterministically transports $x_0$ to $x_1$ and $\hat{x}_0$ to $\hat{x}_1$, it follows that $x_1$ and $\hat{x}_1$ share the same density.
This implies that the resulting distribution of $x_1$ and $\hat{x}_1$ respects the desired equivariance with respect to $g$.
We emphasize that this is only an intuitive explanation; a rigorous proof is provided in Appendix~\ref{sec:theorem1}. 

Since standard Flow Matching uses a Gaussian distribution as $p_0$ by default, the isotropy condition in Theorem~\ref{th:equi} is automatically satisfied. 
As a result, making Flow Policy equivariant reduces to designing an equivariant network $u_\theta$.

\subsubsection{Design of the Equivariant \texorpdfstring{$u_\theta$}{u-theta}} 

To implement the equivariant policy network $u_\theta$, we leverage the \texttt{escnn} library~\citep{cesa2022program}, which supports constructing neural networks that are equivariant to symmetry groups (planar rotations modeled by SO(2) in our case). 
A critical step in using \texttt{escnn} is specifying how each output component transforms under group actions, which requires carefully choosing representations that respect the underlying task symmetries.

In our setting, the policy outputs an absolute 6-DoF end-effector pose with 3D rotation and 3D translation, along with a scalar gripper width to control a robot arm.
To represent the 3D rotation, we adopt the 6D continuous representation~\citep{zhou2019continuity} that encodes the first two rows of a $3 \times 3$ rotation matrix, corresponding to the $x$ and $y$ axes. 
This 6D representation can be seen as three 2D vectors in the $x$-$y$ plane, which transform under SO(2) according to the irreducible representation $\rho_1$. Therefore, the 3D rotation component corresponds to $\rho_1^3$.

For 3D translation, the $x$ and $y$ components transform as a 2D vector under SO(2), again corresponding to $\rho_1$, while the $z$ component remains invariant and is modeled as $\rho_0$. 
The scalar gripper width is also invariant under planar rotation, corresponding to another $\rho_0$.

Combining these components, the action vector at robot time $\tau$, denoted by $x^\tau$, is a 10D vector comprising a 6D rotation representation (first 6 dimensions), a 3D translation vector (next 3 dimensions), and a scalar gripper width (final dimension). The corresponding equivariant representation of the action output is:
\[
g x^\tau = (\rho_1^3 \oplus (\rho_1 \oplus \rho_0) \oplus \rho_0)(g)  x^\tau.
\]
This representation enables $u_\theta$ to produce actions that respect the SO(2) symmetry of the task, ensuring consistent behavior under planar rotations of the scene.

\subsubsection{Network Architecture}
As introduced above, the input of $u_\theta$ is flow time $t$, action sequence $x_t$, and observation $o$.
We set the equivariant group as a finite cyclic subgroup $C_u \in \text{SO}(2)$, and $u$ is the order of the group.
We first use an equivariant observation encoder to map observation $o$ to embeddings $e_o \in \mathbb{R}^{u \times d_o}$ and use an equivariant action encoder to map action sequence $x_t$ to embeddings $e_x \in \mathbb{R}^{u \times d_x}$, where $d_o$ and $d_x$ are the feature dimensions associated with each group element.

The encoded embeddings $e_o, e_x$, along with the timestep $t$, are fed into a core equivariant neural network. This network, together with the observation and action encoders, parameterizes the conditional vector field $u_\theta(t, x_t, o)$. As all components are designed to be equivariant, the entire mapping process from raw inputs to the predicted vector field strictly adheres to $C_u$ symmetry.

\subsubsection{Temporal Consistency}
\label{sec:temporal_consistency}
When generating action sequences, adjacent segments are predicted independently. As a result, the policy may switch between different behavioral modes, leading to inconsistencies in long-term execution.

To address this, we adopt a temporal overlapping strategy similar to \citep{chi2023diffusion}: only the first $n_1$ steps of each predicted sequence are executed, while the remaining $n - n_1$ steps overlap with the subsequent prediction starting from time $\tau + n_1$. 
Long-term consistency can be achieved by generating neighboring action sequences with similar overlap.

To this end, we employ a batched trajectory selection and periodic reset strategy, inspired by IMLE Policy~\citep{rana2025imle}, which balances multi-modal expressivity with temporal coherence.
During inference, we sample $m$ initial noise vectors $\{x_{0,i}\}_{i=1}^m$ from a Gaussian distribution and evolve each through our model to generate $m$ candidate action trajectories $\{x_{1,i}\}_{i=1}^m$. We then select the trajectory whose overlapping segment best matches the previous trajectory in the Euclidean sense:
\begin{align}
\arg\min_{i \in {1,\dots,m}} d\left([x_{\text{pre}}^{\tau+n_1}, \dots, x_{\text{pre}}^{\tau+n}], [x_{1,i}^{\tau+n_1}, \dots, x_{1,i}^{\tau+n}]\right), \nonumber
\end{align}
where we assume the current robot time is $\tau+n_1$, and $x_{\text{pre}}$ denotes the previous action sequence predicted at time $\tau$, where the steps $x_{\text{pre}}^{\tau}, \dots, x_{\text{pre}}^{\tau+n_1-1}$ have already been executed.

To preserve the model’s ability to explore diverse behaviors, we introduce periodic resets: every 10 prediction cycles, we randomly select one trajectory from the batch for execution, instead of the one that minimizes the overlap distance.
This approach improves temporal consistency while maintaining multi-modality, and the batched design ensures minimal overhead in inference time due to parallelization.

\subsection{Acceleration Regularization}
\label{sec:EfficientFlow}
In our experiments, we observe that flow-based policies trained solely with the conditional flow matching objective (Eq.~\ref{eq:CFM}) tend to perform poorly when the number of function evaluations (NFE) is low. This suggests that the learned flow fields are overly curved, requiring more integration steps for accurate trajectory generation.

To address this, we propose an acceleration regularization term that encourages smoother, low-curvature flow trajectories. The underlying intuition is that smoother motion corresponds to smaller second-order derivatives (accelerations) of the trajectory $x_t$. In the extreme case of zero acceleration, the trajectory becomes a straight line.

We augment the training objective as follows:
\begin{equation}
\label{eq:Acc}
\underbrace{\mathbb{E} \left[ \left\| u_\theta(t,x_t) - u(t,x_t \mid x_1) \right\|_2^2 \right]}_{\text{Data Term}} + \lambda\!\!\!   \underbrace{\mathbb{E} \left[ \left\| \frac{\mathrm{d}^2x_t}{\mathrm{d}t^2} \right\|_2^2 \right]}_{\text{Acceleration Penalty}},
\end{equation}
where $\lambda$ controls the trade-off between fidelity to the target velocity field and trajectory smoothness. In practice, we use a time-dependent weighting $\lambda(t) = (1 - t)^2$, which encourages smoother flow at earlier timesteps and prioritizes accuracy as $t \to 1$.

According to Eq.~\ref{eq:ode}, the second derivative term can be rewritten as:
\begin{align}
\mathbb{E} \left[ \left\| \frac{\mathrm{d}^2x_t}{\mathrm{d}t^2} \right\|_2^2 \right]
\approx \frac{1}{(\Delta t)^2} \mathbb{E} \left\| u_\theta(t, x_t) - u_\theta(t + \Delta t, x_{t + \Delta t}) \right\|_2^2, \nonumber
\end{align}
which, however, cannot be directly evaluated, because $x_t$ and $x_{t+\Delta t}$ lie on the same underlying marginal trajectory that is unknown.

To overcome this, we introduce a practical surrogate regularization, which we call the Flow Acceleration Upper Bound (FABO):
\begin{align}
\label{eq:fabo}
\text{FABO} = \mathbb{E} \left\| u_\theta(t, \tilde{x}_t) - u_\theta(t + \Delta t, \tilde{x}_{t + \Delta t}) \right\|_2^2 
\geq  \mathbb{E} \left\| u_\theta(t, x_t) - u_\theta(t + \Delta t, x_{t + \Delta t}) \right\|_2^2, 
\end{align}
when $\Delta t$ is small.
Notably, $\tilde{x}_t$ and $\tilde{x}_{t + \Delta t}$ are sampled from the same conditional trajectory at time $t$ and $t + \Delta t$, which are easy to draw and require no knowledge of the marginal trajectory.

In essence, FABO minimizes an upper bound on the true acceleration penalty, serving as a tractable and effective proxy. A formal proof for Eq.~\ref{eq:fabo} is provided in Appendix~\ref{sec:fabo}. 

\begin{table*}[t] 
\centering 
\setlength{\tabcolsep}{1mm} 
\sisetup{round-mode=places,round-precision=0,table-align-text-post=false} 
\begin{center}
\small
\begin{tabular}{llc 
                *{18}{S[table-format=2,table-number-alignment = center]} 
                 }
\toprule
\multirow{2}{*}{Method}& \multirow{2}{*}{Obs} & \multirow{2}{*}{NFE} & \multicolumn{3}{c}{\textbf{Stack D1}}  & \multicolumn{3}{c}{\textbf{Square D2}} & \multicolumn{3}{c}{\textbf{Threading D2}}& \multicolumn{3}{c}{\textbf{Stack Three D1}} & \multicolumn{3}{c}{\textbf{Coffee D2}} & \multicolumn{3}{c}{\textbf{3 Pc. Asm. D2}} \\
\cmidrule(lr){4-6} \cmidrule(lr){7-9} \cmidrule(lr){10-12} \cmidrule(lr){13-15} \cmidrule(lr){16-18} \cmidrule(lr){19-21}
 &  & & {100} & {200} & {1000}& {100} & {200} & {1000}& {100} & {200} & {1000}& {100} & {200} & {1000}& {100} & {200} & {1000}& {100} & {200} & {1000}\\ 
\midrule
\multirow{3}{*}{Ours} & \multirow{3}{*}{RGB} & 1    & \textbf{94}   & \textbf{100}  & \textbf{100}    & 21   & \textbf{45}   & 67   & \textbf{31}   & 36   & 49   & 48   & 73   & 92 & 65   & \textbf{81}   & 79   & 11   & 35   & 60   \\
& & 3               & 88   & \textbf{100}  & \textbf{100}    & 20   & \textbf{45}  & \textbf{71}   & \textbf{31}   & \textbf{43}  & 53  & 49   & 76   & 94  & 66   &  80   & \textbf{84}   & 11    & 38   & 69   \\
& & 5               & 87   & \textbf{100}  & \textbf{100}    & 22   & 43   & \textbf{71}   & \textbf{31}   & 41   & 58 & 50   & \textbf{79}   & 93   & \textbf{67}   & 79   & 83   & 11   &\textbf{42}    & \textbf{71}   \\

\midrule
EquiDiff & RGB & 100                                    & 93 & \textbf{100}  & \textbf{100}  & \textbf{25} & 41.3 & 60   & 22   & 40   & \textbf{59}  & \textbf{55} & 77 & \textbf{96}   & 60   & 79 & 76   & \textbf{15} & 39 & 69 \\
DP-C     & RGB & 100                                    & 76   & 97   & \textbf{100}  & 8    & 19   & 46   & 17   & 35   &\textbf{59}   & 38   & 72   & 94   & 44   & 66   & 79   & 4    & 6    & 30   \\
DP-T     & RGB & 100                                    & 51   & 83   & 99    & 5    & 11   & 45   & 11   & 18   & 41   & 17   & 41   & 84  & 47   & 61   & 75   & 1    & 4    & 43   \\
DP3      & PCD & 10                                     & 69   & 87   & 99    & 7    & 6    & 19   & 12   & 23   & 40  & 7    & 23   & 65   & 34   & 45   & 69   & 0    & 1    & 3    \\
ACT      & RGB & 1                                   & 35   & 73   & 96    & 6    & 18   & 49   & 10   & 21   & 35    & 6    & 37   & 78 & 19   & 33   & 64   & 0    & 3    & 24   \\
\midrule[\heavyrulewidth] 
\multirow{2}{*}{Method}& \multirow{2}{*}{Obs} & \multirow{2}{*}{NFE} & \multicolumn{3}{c}{\textbf{Hammer Cln. D1}} & \multicolumn{3}{c}{\textbf{Mug Cln. D1}} & \multicolumn{3}{c}{\textbf{Kitchen D1}} & \multicolumn{3}{c}{\textbf{Pick Place D0}}  & \multicolumn{3}{c}{\textbf{Nut Asmn.D0}}& \multicolumn{3}{c}{\textbf{Coffee Pre. D1}} \\ 
\cmidrule(lr){4-6} \cmidrule(lr){7-9} \cmidrule(lr){10-12} \cmidrule(lr){13-15} \cmidrule(lr){16-18} \cmidrule(lr){19-21}
 &  & &{100} & {200} & {1000}& {100} & {200} & {1000}& {100} & {200} & {1000}& {100} & {200} & {1000}& {100} & {200} & {1000}& {100} & {200} & {1000}\\
\midrule
\multirow{3}{*}{Ours} & \multirow{3}{*}{RGB} & 1    & \textbf{75}   & \textbf{75}   & 77   & \textbf{50}   & 65   & 67   & 66   & 78   & 81  &37  & 50   & 67   & 59   & 83   & 94    & 75   & 74   & 70   \\

& & 3                                                    & 72  & \textbf{75}   & 84   & \textbf{50}   & 65   & \textbf{70}   & \textbf{73}   & 81   & 81  & 41 & 62 & 86   & 61   & 86   & 96  & 81   & 81   & \textbf{89}   \\
& & 5                                                    & 74   & \textbf{75}   & 83   & \textbf{50}  & \textbf{68}  & \textbf{70}   & \textbf{73}   & 81   & 83   & 41 & 66 & 87  & 62   & \textbf{87}   & \textbf{98}    & \textbf{83}   & 82   & 87   \\
\midrule
EquiDiff & RGB & 100                                    & 65 & 63 & 77 & \textbf{50} & 64   & 67 & 67   & 77 & 81 & \textbf{42}   & \textbf{74} & \textbf{92}   & \textbf{74}   & 85  & 94   & 77 & \textbf{83} & 85 \\
DP-C     & RGB & 100                                    & 52   & 59   & 73   & 43   & 59   & 65   & 67   & \textbf{85}   & 87  & 35   & 65   & 83 & 55   & 68   & 83      & 65   & 62   & 58   \\
DP-T     & RGB & 100                                    & 48   & 60   & 76   & 30   & 43   & 63   & 54   & 75   & 81   & 15   & 37   & 50  & 31   & 32   & 46    & 38   & 51   & 76   \\
DP3      & PCD & 10                                     & 54   & 71   & \textbf{87}   & 21   & 33   & 53   & 45   & 71   & \textbf{91}   & 12   & 15   & 34   & 16   & 24   & 58   & 10   & 22   & 63   \\
ACT      & RGB & 1                                    & 38   & 54   & 71   & 23   & 31   & 56   & 37   & 61   & 87   & 7    & 17   & 50  & 42   & 64   & 84    & 32   & 46   & 65   \\
\bottomrule
\end{tabular}
\end{center}
\caption{Comparison against SOTA. We report the success rates of 12 MimicGen~\citep{mandlekar2023mimicgen} tasks using 100, 200, and 1000 demonstrations, respectively. Results averaged over three seeds. The results of the baseline methods are directly cited from EquiDiff~\citep{wang2024equivariant}.} 
\label{tab:success_rates} 
\end{table*}

\section{Experiments}
\label{sec:experiments}

\subsection{Implementation Details}

We evaluate EfficientFlow on 12 tasks from the MimicGen benchmark~\citep{mandlekar2023mimicgen}.
These tasks span a wide range of difficulties, time horizons, and object arrangements, providing a comprehensive testbed for assessing policy performance across diverse robotic manipulation scenarios.

Notably, the agent view camera in MimicGen is not positioned orthogonally to the workspace but rather provides an agent-centric perspective. While the resulting image rotations may not perfectly align with the true object rotations, potentially impacting the performance of equivariant networks, this setup more closely mirrors real-world scenarios where state information is acquired from non-ideal viewpoints and thus offers a more rigorous test of policy generalization.

Our work aims for effective and efficient robot control in embodied AI.
Thus, we compare EfficientFlow against strong baselines in this field, including EquiDiff~\citep{wang2024equivariant}, ACT~\citep{zhao2023learning}, and two variants of Diffusion Policy~\citep{chi2023diffusion}: the CNN-based DP-C and the Transformer-based DP-T.
For fair comparison, all baseline methods use the same input as EfficientFlow: RGB images from both the agent-view and wrist-mounted cameras.
In addition to the RGB-based methods, we further compare with DP3~\citep{Ze2024DP3}, which utilizes 3D point cloud information. 
All policies employ absolute pose estimation for control. 
We evaluate EfficientFlow with 1, 3, and 5 NFE, while for baseline methods, we adhere to their original configurations.

\subsection{Quantitative Comparison}

\subsubsection{Sampling Efficiency}
EfficientFlow demonstrates notable advantages in inference speed and sampling efficiency. 
As shown in Table~\ref{tab:AverageSuccessRate}, 
it achieves an average inference time of only 12.22 ms under a 1-NFE setting, offering an approximately 56.1$\times$ speedup over EquiDiff.
Even with a more computationally demanding 5-NFE inference, EfficientFlow remains approximately 19.9$\times$ faster than EquiDiff on average.
This significant gain in efficiency is critical for real-time robotic control, where rapid response is essential, and enables inference frequencies of up to 81.8 Hz with single-step EfficientFlow.

\subsubsection{Data Efficiency}

As shown in Table~\ref{tab:success_rates}, 
under the data-limited setting of only 100 demonstrations, 
EfficientFlow not only enables significantly faster inference but also achieves a strong policy success rate.
Across the 12 test tasks, EfficientFlow outperforms EquiDiff in 7 of them. 
For the remaining 5 tasks, except the Nut Assembly D0 task, the performance gap between EfficientFlow and EquiDiff is within 5 percentage points. 
These results indicate that EfficientFlow can match or even surpass the performance of the SOTA baseline method while drastically reducing inference latency.
When trained with 200 demonstrations, EfficientFlow achieves 98.4\% of the success rate of the DP-C method trained with 1000 demonstrations, while surpassing the average success rates of DP-T, DP3, and ACT. 
This remarkable performance underscores its exceptional data efficiency and strong generalization under limited supervision.

\begin{table}[t]
  \centering
  \small
  \setlength{\tabcolsep}{2.5mm}
  \begin{tabular}{l c c c c c}
    \toprule
    \multirow{2}{*}{{Method}} 
      & \multirow{2}{*}{{NFE}} 
      & \multirow{2}{*}{{Runtime (ms)}} 
      & \multicolumn{3}{c}{\textbf{Average Success Rate (\%)}} \\
    \cmidrule(lr){4-6}
      & & & {100} & {200} & {1000} \\
    \midrule

    \multirow{3}{*}{Ours} 
        & 1 & 12.22 & 52.61 & 66.18 & 75.25 \\
        & 3 & 22.59 & 53.49 & \textbf{69.33} & \textbf{81.36} \\
        & 5 & 34.45 & \textbf{54.18} & \textbf{70.26} & \textbf{81.99} \\
    \midrule
    EquiDiff & 100 & 685.92 & 53.77 & 68.59 & 79.69 \\
    DP-C     & 100 & 542.96 & 42.00 & 57.75 & 71.42 \\
    DP-T     & 100 & 497.53 & 29.00 & 43.00 & 64.92 \\
    DP3      & 10  &  53.83 & 23.92 & 35.08 & 56.75 \\
    ACT      & 1   &  12.51 & 21.33 & 38.17 & 63.25 \\
    \bottomrule
  \end{tabular}
  
  \caption{Average success rates and inference time of EfficientFlow and baselines across 12 MimicGen tasks.}
  \label{tab:AverageSuccessRate}
\end{table}

\begin{table}[t]
  \centering  
    \small
  \begin{tabular}{l|cccccccccccc|c}
    \toprule
    Method & Sk  & Sq & Th & S3& Cf & 3P & Hm & Mu & Ki & PP  & Nu& CP & Avg\\
    \midrule
    Ours   & \textbf{20} & 60 & \textbf{20} & \textbf{40} & \textbf{30} & \textbf{40} & \textbf{10} & 30& \textbf{30} & \textbf{40} & 40 & \textbf{20} & \textbf{31.7} \\
    EquiDiff & 70 & 80 & 40 & 50 & 40 & 70 & 50 & 40 & \textbf{30} & 60 & 50 & 40 & 51.7\\
    NoAcc    & 30 & \textbf{40} & \textbf{20} & 60 & \textbf{30} & 60 & 20 & \textbf{10} & \textbf{30} & 50 & \textbf{30}& 40 & 35.0    \\ 
    \bottomrule
  \end{tabular}
    \caption{Minimum training epochs required for EfficientFlow and EquiDiff to reach 50\,\% of their final maximum success rate (MimicGen, 100 demonstrations, evaluated every 10 epochs with a fixed seed).}
     \label{tab:50PercentofMaxRate}
\end{table}

\begin{table}[t]
\centering
 \small

\begin{tabular}{l|cccccccccccc|c}
\toprule
Method & Sk  & Sq & Th & S3& Cf & 3P & Hm & Mu & Ki & PP  & Nu& CP & Avg\\
\midrule
Ours   & 94  & 21 & \textbf{31} & \textbf{48} & \textbf{65} & \textbf{11} & \textbf{75} & \textbf{50} & \textbf{66}  & \textbf{37} & \textbf{59} & \textbf{75} & \textbf{52.6} \\
NoAcc    & 88  & 16 & 24& 44 & 56 & 10 & 56 & 28 & 42  & 16.5& 28 & 62 & 39.3 \\
NonEqui  & 88  & 12 & 12 & 14 & 52 & 0 & 66 & 38 & 54  & 17.5& 43 & 56 & 37.7\\
EquiCFM  &88  &8 &22 &44& 60 &8 &54 & 36 &\textbf{66} &19  &40&40 &40.4\\
EquiMF  &\textbf{96}  &\textbf{22} &26 &34& 50 &8 & 58 &\textbf{50} &62  &26 &51 & 72 &46.3\\

\bottomrule
\end{tabular}
\caption{Ablation study on 12 MimicGen tasks using 100 demonstrations.}
\label{tab:AblationStudies}
\end{table}

\subsubsection{More Analysis of the Performance}
As shown in Table~\ref{tab:AverageSuccessRate}, the average success rates across all tasks further highlight the advantage of EfficientFlow.
First, the average performance of EfficientFlow exhibits an upward trend as the NFE increases; with sufficient data, multi-step inference can better capture the conditional action distribution to achieve higher success rates. 
More importantly, across all dataset sizes, EfficientFlow consistently exceeds the average success rate of EquiDiff~\citep{wang2024equivariant} while requiring dramatically fewer inference steps.
We attribute the advantage of EfficientFlow to two key design choices: the strong inductive biases introduced by the equivariant architecture, and the acceleration regularization that stabilizes the action sampling trajectories.
These factors enable the model to more efficiently learn key task structures and robust dynamics representations from limited demonstrations.

\subsubsection{Learning Efficiency}

To quantitatively evaluate the learning efficiency of EfficientFlow, we measure the minimum number of training epochs required to reach 50\% of the final peak success rate. 
As shown in Table~\ref{tab:50PercentofMaxRate}, both EfficientFlow and its variant without acceleration regularization (NoAcc) require substantially fewer training epochs than EquiDiff to reach 50\% of their maximum success rate. 
Notably, in the Hammer Cleanup D1 task, EfficientFlow requires only one-fifth of the epochs needed by EquiDiff. 
These results demonstrate the improved learning efficiency and stronger optimization dynamics of our equivariant flow-based framework.
The acceleration constraint further improves convergence speed, as evidenced by the faster learning of EfficientFlow compared to NoAcc.
These results indicate that EfficientFlow can extract essential policy information from demonstrations more rapidly, which reflects its superior learning dynamics, 
enabling it to reach target performance levels significantly faster than baseline methods.

\subsection{Ablation Study}
To disentangle the contributions of the equivariant architecture and the acceleration regularization, we conduct comprehensive ablation studies across 12 MimicGen tasks using 100 demonstrations in Table~\ref{tab:AblationStudies}. 
We evaluate four key variants:
1) NoAcc, which removes the acceleration term and is trained solely with the Conditional Flow Matching loss ($L_{\text{CFM}}$ in Eq.~\ref{eq:CFM});
2) NoEqui, which discards the equivariant architecture in favor of a non-equivariant backbone similar to Diffusion Policy~\citep{chi2023diffusion} while retaining the acceleration constraint;
3) EquiCFM, which combines the equivariant network with Consistency Flow Matching~\citep{yang2024consistency}; and
4) EquiMF, which integrates the equivariant network with MeanFlow~\citep{geng2025mean}.

The results reveal two clear trends. 
First, both our proposed components, including equivariance and acceleration regularization, substantially and independently improve performance. Removing either one leads to a consistent drop in success rate, demonstrating their complementary roles: the equivariant structure provides strong inductive biases for learning symmetric behaviors, while the acceleration term stabilizes sampling trajectory learning.

Second, to further investigate the benefit of our acceleration regularization, we compare EfficientFlow against other efficient one-step flow matching variants. 
By replacing our formulation with Consistency Flow Matching~\citep{yang2024consistency} and MeanFlow~\citep{geng2025mean} while keeping the equivariant architecture fixed, we observe that EfficientFlow achieves higher overall success rates across tasks. This suggests that our acceleration-regularized formulation not only stabilizes training but also leads to more accurate and robust policy generation.

\section{Conclusion}

We introduce EfficientFlow, a theory-grounded generative policy learning framework that effectively balances inference speed and data efficiency. By leveraging equivariant flow matching and acceleration regularization, this work provides a principled approach for learning robust visuomotor policies while ensuring strong generalization and efficient learning. 
The ultra-fast inference and strong data efficiency of EfficientFlow highlight its potential as a practical and high-performance solution for real-world embodied AI systems.

\section{Acknowledgments}
This research was supported by the National Natural Science Foundation of China (62302385).

\clearpage
\appendix

\section*{Appendix}
\setcounter{theorem}{0}
\label{appendix}

In this appendix, we first present the proof of Theorem~1 of the main paper in Section~\ref{sec:theorem1}, followed by the proof of the Flow Acceleration Upper Bound in Section~\ref{sec:fabo}. In Section~\ref{sec:err_term}, we analyze the error term generated by FABO.
We complement the main results in our paper with standard deviation in Section~\ref{sec:std}.
Additional details about the simulation environment and algorithm implementation are provided in Sections~\ref{sec:simulation} and~\ref{Implementation Details}, respectively.

\section{Proof of Theorem 1}\label{sec:theorem1}
\counterwithout{theorem}{section} 
\begin{theorem}\label{th:equi-appendix}
Let \( G \) be a transformation group acting on both the observation space and the action space. Suppose the initial distribution \( p_0 \) is isotropic, i.e., \( p_0(g x) = p_0(x) \) for all \( g \in G \), and the velocity network \( u_\theta(t, x | o) \) is equivariant, i.e., \(u_{\theta}(t,gx|go)=gu_{\theta}(t,x|o) \) for all \(g\in G\).
Then the induced conditional distribution at time \( t \), \(X_t|_{O=o} \), given by the flow ODE 
\(\frac{dx_t}{dt}=u_{\theta}(t,x_t|o),x_0\sim p_0\), satisfies
\begin{align}\label{eq:theorem-appendix}
    X_t \big|_{O = go} \stackrel{d}{=} g \left( X_t \big|_{O = o} \right), \quad t\in [0,1]
\end{align}
i.e., the output distribution is equivariant under the group action.
\end{theorem}

\begin{proof}
Let \(\Phi_t(x_0| o)\) be the solution of the ODE at time \(t\) with the initial value \(x_0\), conditioned on \(o\), so \(x_t=\Phi_t(x_0| o)\). We first show $\Phi_t$ is equivariant.
Let $\hat{x}_t = g\Phi_t(x_0| o)$. Its initial condition is $\hat{x}_0 = g\Phi_0(x_0| o) = gx_0$. Its dynamics are:
\begin{equation}
\frac{d}{dt}\hat{x}_t = g \frac{d}{dt}\Phi_t(x_0| o) = g u_{\theta}(t, \Phi_t|o) = u_{\theta}(t, g\Phi_t| go) = u_{\theta}(t, \hat{x}_t| go).    
\end{equation}
Since $\hat{x}_t$ and $\Phi_t(gx_0| go)$ share the same initial condition and ODE dynamics, by uniqueness~\citep{perko2013differential}, we have $g\Phi_t(x_0| o) = \Phi_t(gx_0| go)$.
Using this equivariance, we can establish the main result.
By definition and the equivariance property, we have the following identity for the random variables:
\begin{equation}
g(X_t|_{O=o}) = g\Phi_t(X_0| o) = \Phi_t(gX_0| go)
\label{ap:eq3}
\end{equation}
This first equality holds by definition of the flow, and the second is the random variable identity derived from the deterministic equivariance of the flow shown above. \\
Now, since the initial distribution \(p_0\) is isotropic, i.e., $X_0 \stackrel{d}{=} gX_0$, applying the same deterministic function $\Phi_t(\cdot| go)$ to both sides preserves the distributional equality. Therefore,
\begin{equation}
\Phi_t(gX_0| go) \stackrel{d}{=} \Phi_t(X_0| go)
\label{ap:eq4}
\end{equation}
Combining \eqref{ap:eq3} and \eqref{ap:eq4}, we have a chain of equalities:
\begin{equation}
g(X_t|_{O=o}) = \Phi_t(gX_0| go) \stackrel{d}{=} \Phi_t(X_0| go)
\end{equation}
By definition, the random variable $\Phi_t(X_0| go)$ is the same as $X_t|_{O=go}$. Thus, we can conclude:
\begin{equation}
g(X_t|_{O=o}) \stackrel{d}{=} X_t|_{O=go}    
\end{equation}
This completes the proof.
\end{proof}
\newpage

\section{Proof of Flow Acceleration Upper Bound (FABO)}\label{sec:fabo}
\begin{theorem}
Assume that $u(t,x)$ is twice continuously differentiable with bounded second derivatives. 
For any $x_t \sim p_t$, we are interested in two trajectories passing through it:
the optimal marginal trajectory denoted as $x_t$ and the linear conditional trajectory denoted as $\tilde{x}_t = (1-t)\tilde{x}_0+t \tilde{x}_1$. 
Then, when \(\Delta t\) is small enough,
\begin{equation}
\mathbb{E}\|u(t,x_t)-u(t+\Delta t,x_{t+\Delta t})\|_2^2\le \mathbb{E}\|u(t,\tilde{x}_t)-u(t+\Delta t,\tilde{x}_{t+\Delta t})\|_2^2 .
\label{ap:eq7}
\end{equation}
\end{theorem}

\begin{proof}
Since two trajectories intersect at the same state at time \(t\), we have $x_t=\tilde{x}_t$.
\begin{align}
&\mathbb{E}\|u(t,x_t)-u(t+\Delta t,x_{t+\Delta t})\|_2^2 \nonumber\\
=&\mathbb{E}\|u(t,\tilde{x}_t)-u(t,\tilde{x}_t)-\frac{\partial{u}}{\partial t}\Delta t-\frac{\partial{u}}{\partial x}u^{\text{opt}}(t,\tilde{x}_t)\Delta t+o(\Delta t)\|_2^2 & \text{(Taylor Expansion)}\nonumber\\
=&\mathbb{E}\|\frac{\partial{u}}{\partial t}\Delta t+\frac{\partial{u}}{\partial x}u^{\text{opt}}(t,\tilde{x}_t)\Delta t+o(\Delta t)\|_2^2  \nonumber \\
=&\mathbb{E}\|\frac{\partial{u}}{\partial t}\Delta t+\frac{\partial{u}}{\partial x}\mathbb{E}[\tilde{x}_1-\tilde{x}_0|\tilde{x}_t]\Delta t+o(\Delta t)\|_2^2 &\text{(Eq.2 of \citep{liu2022flow})} \nonumber\\
=&\mathbb{E}\Biggl[\|\frac{\partial{u}}{\partial t}\Delta t\|_2^2+2(\frac{\partial{u}}{\partial t})^T\frac{\partial{u}}{\partial x}\mathbb{E}[\tilde{x}_1-\tilde{x}_0|\tilde{x}_t]\Delta t^2+ \|\frac{\partial{u}}{\partial x}\mathbb{E}[\tilde{x}_1-\tilde{x}_0|\tilde{x}_t]\Delta t\|_2^2\Biggr]+o(\Delta t^2) \nonumber\\
=&\mathbb{E}\Biggl[\mathbb{E}\Biggl[\|\frac{\partial{u}}{\partial t}\Delta t\|_2^2+2(\frac{\partial{u}}{\partial t})^T\frac{\partial{u}}{\partial x}(\tilde{x}_1-\tilde{x}_0)\Delta t^2+ \|\frac{\partial{u}}{\partial x}\mathbb{E}[\tilde{x}_1-\tilde{x}_0|\tilde{x}_t]\Delta t\|_2^2\bigg| \tilde{x}_t\Biggr]\Biggr]+o(\Delta t^2) \nonumber\\
=&\mathbb{E}\Biggl[\|\frac{\partial{u}}{\partial t}\Delta t\|_2^2+2(\frac{\partial{u}}{\partial t})^T\frac{\partial{u}}{\partial x}(\tilde{x}_1-\tilde{x}_0)\Delta t^2+ \|\frac{\partial{u}}{\partial x}\mathbb{E}[\tilde{x}_1-\tilde{x}_0|\tilde{x}_t]\Delta t\|_2^2\Biggr]+o(\Delta t^2) &\text{(Total Expectation)} \nonumber\\
=&\mathbb{E}\Biggl[\|\frac{\partial{u}}{\partial t}\Delta t\|_2^2+2(\frac{\partial{u}}{\partial t})^T\frac{\partial{u}}{\partial x}(\tilde{x}_1-\tilde{x}_0)\Delta t^2+ \|\frac{\partial{u}}{\partial x}(\tilde{x}_1-\tilde{x}_0)\Delta t\|_2^2\Biggr]+o(\Delta t^2) \nonumber\\
&-\mathbb{E}\Biggl[\|\frac{\partial{u}}{\partial x}(\tilde{x}_1-\tilde{x}_0)\Delta t\|_2^2 -\|\frac{\partial{u}}{\partial x}\mathbb{E}[\tilde{x}_1-\tilde{x}_0|\tilde{x}_t]\Delta t\|_2^2\Biggr] \nonumber\\
=&\mathbb{E}\Biggl[\|\frac{\partial{u}}{\partial t}\Delta t\|_2^2+2(\frac{\partial{u}}{\partial t})^T\frac{\partial{u}}{\partial x}(\tilde{x}_1-\tilde{x}_0)\Delta t^2+ \|\frac{\partial{u}}{\partial x}(\tilde{x}_1-\tilde{x}_0)\Delta t\|_2^2\Biggr]+o(\Delta t^2) \nonumber\\
&-\mathbb{E}\Biggl[\mathrm{tr}\Biggl[\mathrm{Var}\Biggl[\frac{\partial{u}}{\partial x}(\tilde{x}_1-\tilde{x}_0)\Delta t\Bigg| \tilde{x}_t\Biggr]\Biggr]\Biggr] \nonumber\\
= &\mathbb{E}\|\frac{\partial{u}}{\partial t}\Delta t+\frac{\partial{u}}{\partial x}(\tilde{x}_1-\tilde{x}_0)\Delta t+o(\Delta t)\|_2^2 -\mathbb{E}\Biggl[\mathrm{tr}\Biggl[\mathrm{Var}\Biggl[\frac{\partial{u}}{\partial x}(\tilde{x}_1-\tilde{x}_0)\Delta t\Bigg| \tilde{x}_t\Biggr]\Biggr]\Biggr]\\ 
=&\mathbb{E}\|u(t,\tilde{x}_t)-u(t+\Delta t,\tilde{x}_{t+\Delta t})\|_2^2 + o(\Delta t^2)-\mathbb{E}\Biggl[\mathrm{tr}\Biggl[\mathrm{Var}\Biggl[\frac{\partial{u}}{\partial x}(\tilde{x}_1-\tilde{x}_0)\Bigg| \tilde{x}_t\Biggr]\Biggr]\Biggr](\Delta t^2) .
\end{align}
When \(\Delta t\) is small enough, we have
\begin{equation}
\mathbb{E}\|u(t,x_t)-u(t+\Delta t,x_{t+\Delta t})\|_2^2\le \mathbb{E}\|u(t,\tilde{x}_t)-u(t+\Delta t,\tilde{x}_{t+\Delta t})\|_2^2 
\end{equation}
\end{proof}
\newpage
\section{Analysis of the Error Term Generated by FABO}\label{sec:err_term}
\begin{lemma} \label{lem:trace_bound}
Let $A$ and $B$ be two $n \times n$ positive semidefinite (PSD) matrices ($A, B \succeq 0$). Let $\lambda_{\min}(B)$ and $\lambda_{\max}(B)$ denote the minimum and maximum eigenvalues of $B$, respectively. Then, the trace of their product is bounded as follows:
$$
\lambda_{\min}(B) \cdot \operatorname{tr}(A) \le \operatorname{tr}(AB) \le \lambda_{\max}(B) \cdot \operatorname{tr}(A)
$$
\end{lemma}
\begin{proof}
Let the spectral decomposition of $A \succeq 0$ be 
$$A=\lambda_1\xi_1\xi_1^T+\cdots+\lambda_n\xi_n\xi_n^T,\|\xi_i\|_2^2=1,\forall1\le i\le n$$
Then
$$\mathrm{tr}(AB)=\sum_{i=1}^{n}\lambda_i\mathrm{tr}(\xi_i\xi_i^TB)=\sum_{i=1}^{n}\lambda_i\frac{\xi_i^TB\xi_i}{\xi_i^T\xi_i}\in[\lambda_{\min}(B)\sum_{i=1}^{n}\lambda_i,\lambda_{\max}(B)\sum_{i=1}^{n}\lambda_i]$$
So
$$
\lambda_{\min}(B) \cdot \operatorname{tr}(A) \le \operatorname{tr}(AB) \le \lambda_{\max}(B) \cdot \operatorname{tr}(A)
$$
\end{proof}
\begin{theorem}
We assume that $\forall a\in \mathbb{R}^n,\|a\|_2^2=1$, $0<\mu_1\le\mathrm{Var}[a^T(\tilde{x}_1-\tilde{x}_0)|\tilde{x}_t]\le\mu_2$. This assumption, holding for all $\tilde{x}_t$, ensures the variance of any linear projection of the $\tilde{x}_1-\tilde{x}_0$ is uniformly bounded: it remains non-deterministic even given $\tilde{x}_t$ and finite. Then we have:
$$\mu_1\mathbb{E}\|\frac{\partial u}{\partial x}\|_F^2\le \mathbb{E}\Biggl[\mathrm{tr}\Biggl[\mathrm{Var}\Biggl[\frac{\partial{u}}{\partial x}(\tilde{x}_1-\tilde{x}_0)\Bigg| \tilde{x}_t\Biggr]\Biggr]\Biggr]\le\mu_2\mathbb{E}\|\frac{\partial u}{\partial x}\|_F^2$$

\end{theorem}
\begin{proof}
\begin{align}
&\mathbb{E}\Biggl[\mathrm{tr}\Biggl[\mathrm{Var}\Biggl[\frac{\partial{u}}{\partial x}(\tilde{x}_1-\tilde{x}_0)\Bigg| \tilde{x}_t\Biggr]\Biggr]\Biggr]\\
=&\mathbb{E}\Biggl[\mathrm{tr}\Biggl[ \frac{\partial{u}}{\partial x}\text{Var}[\tilde{x}_1-\tilde{x}_0|\tilde{x}_t][\frac{\partial{u}}{\partial x}]^T\Biggr]\Biggr] \\
=&\mathbb{E}\Biggl[\mathrm{tr}\Biggl[ [\frac{\partial{u}}{\partial x}]^T \frac{\partial{u}}{\partial x}\text{Var}[\tilde{x}_1-\tilde{x}_0|\tilde{x}_t]\Biggr]\Biggr]
\end{align}
Under the assumption, $\text{Var}[\tilde{x}_1-\tilde{x}_0|\tilde{x}_t]$ is a positive definite matrix whose eigenvalues lie in the interval $[\mu_1,\mu_2]$. Meanwhile, $ [\frac{\partial{u}}{\partial x}]^T \frac{\partial{u}}{\partial x}$ is also positive semidefinite (PSD). Using Lemma \ref{lem:trace_bound},
$$\mathbb{E}\Biggl[\lambda_{\min}\mathrm{tr}\Biggl[ [\frac{\partial{u}}{\partial x}]^T \frac{\partial{u}}{\partial x}\Biggr]\Biggr]\le\mathbb{E}\Biggl[\mathrm{tr}\Biggl[ [\frac{\partial{u}}{\partial x}]^T \frac{\partial{u}}{\partial x}\text{Var}[\tilde{x}_1-\tilde{x}_0|\tilde{x}_t]\Biggr]\Biggr]\le\mathbb{E}\Biggl[\lambda_{\max}\mathrm{tr}\Biggl[ [\frac{\partial{u}}{\partial x}]^T \frac{\partial{u}}{\partial x}\Biggr]\Biggr]$$
Since 
$$\mathbb{E}\Biggl[\mathrm{tr}\Biggl[ [\frac{\partial{u}}{\partial x}]^T \frac{\partial{u}}{\partial x}\Biggr]\Biggr]=\mathbb{E}\|\frac{\partial u}{\partial x}\|_F^2$$
and
$$\mu_1\le \lambda_{\min}\le\lambda_{\max}\le\mu_2,$$
we have
$$\mu_1\mathbb{E}\|\frac{\partial u}{\partial x}\|_F^2\le \mathbb{E}\Biggl[\mathrm{tr}\Biggl[ [\frac{\partial{u}}{\partial x}]^T \frac{\partial{u}}{\partial x}\text{Var}[\tilde{x}_1-\tilde{x}_0|\tilde{x}_t]\Biggr]\Biggr]\le\mu_2\mathbb{E}\|\frac{\partial u}{\partial x}\|_F^2$$
\end{proof}
\section{Standard Deviation of Evaluation Results}\label{sec:std}
The main evaluation experiments are repeated with three random seeds.
Due to space limitations, only the average results are reported in the main paper, while the corresponding standard deviations are provided in Table~\ref{table1}.

\begin{table}[H]
\centering
\small 
\setlength{\tabcolsep}{4pt} 
\resizebox{1\linewidth}{!}{
\begin{tabular}{@{}llccccccccccccc@{}} 
\toprule
&  &  & \multicolumn{3}{c}{Stack D1} & \multicolumn{3}{c}{Stack Three D1} & \multicolumn{3}{c}{Threading D2} & \multicolumn{3}{c}{Square D2} \\
\cmidrule(lr){4-6} \cmidrule(lr){7-9} \cmidrule(lr){10-12} \cmidrule(lr){13-15}
Method &Obs & NFE& 100 & 200 & 1000 & 100 & 200 & 1000 & 100 & 200 & 1000 & 100 & 200 & 1000 \\
\midrule
\multirow{3}{*}{Ours} & \multirow{3}{*}{RGB} & 1 & 94.0$\pm$1.6 & 100.0$\pm$0.0 & 100.0$\pm$0.0 & 48.0$\pm$0.0 & 73.3$\pm$5.0 & 92.0$\pm$0.0 & 31.3$\pm$1.9 & 36.0$\pm$1.6 & 48.7$\pm$2.5 & 20.7$\pm$1.9 & 44.7$\pm$1.9 & 67.3$\pm$0.9 \\
& & 3 & 88.0$\pm$1.6 & 100.0$\pm$0.0 & 100.0$\pm$0.0 & 49.3$\pm$3.4 & 76.0$\pm$4.3 & 94.0$\pm$0.0 & 30.7$\pm$5.2 & 42.7$\pm$2.5 & 53.3$\pm$2.5 & 20.0$\pm$1.6 & 45.3$\pm$2.5 & 70.7$\pm$0.9 \\
& & 5 & 86.7$\pm$2.5 & 100.0$\pm$0.0 & 100.0$\pm$0.0 & 50.0$\pm$2.8 & 78.7$\pm$5.0 & 93.3$\pm$1.9 & 30.7$\pm$1.9 & 41.3$\pm$1.9 & 58.0$\pm$2.8 & 22.0$\pm$1.6 & 43.3$\pm$1.9 & 71.3$\pm$0.9 \\
\midrule
EquiDiff & RGB & 100 & 93.3$\pm$0.7 & 100.0$\pm$0.0 & 100.0$\pm$0.0 & 54.7$\pm$5.2 & 77.3$\pm$1.8 & 96.0$\pm$1.2 & 22.0$\pm$1.2 & 40.0$\pm$1.2 & 59.3$\pm$1.8 & 25.3$\pm$8.7 & 41.3$\pm$9.8 & 60.0$\pm$4.2 \\
DP-C & RGB & 100 & 76.0$\pm$4.0 & 97.3$\pm$0.7 & 100.0$\pm$0.0 & 38.0$\pm$0.0 & 72.0$\pm$2.0 & 94.0$\pm$1.2 & 17.3$\pm$1.8 & 35.3$\pm$1.3 & 58.7$\pm$0.7 & 8.0$\pm$1.2 & 19.3$\pm$5.3 & 46.0$\pm$7.2 \\
DP-T & RGB & 100 & 51.3$\pm$1.8 & 82.7$\pm$0.7 & 98.7$\pm$0.7 & 16.7$\pm$0.7 & 41.3$\pm$2.9 & 84.0$\pm$1.2 & 10.7$\pm$0.7 & 18.0$\pm$1.2 & 40.7$\pm$0.7 & 4.7$\pm$1.8 & 11.3$\pm$2.4 & 44.7$\pm$4.7 \\
DP3 & PCD & 10 & 69.3$\pm$3.7 & 86.7$\pm$4.7 & 99.3$\pm$0.7 & 7.3$\pm$0.7 & 22.7$\pm$3.7 & 65.3$\pm$1.8 & 12.0$\pm$3.1 & 23.3$\pm$3.3 & 40.0$\pm$2.0 & 6.7$\pm$0.7 & 6.0$\pm$0.0 & 19.3$\pm$3.3 \\
ACT & RGB & 1 & 34.7$\pm$0.7 & 72.7$\pm$7.7 & 96.0$\pm$1.2 & 6.0$\pm$2.3 & 36.7$\pm$2.7 & 78.0$\pm$1.2 & 10.0$\pm$1.2 & 20.7$\pm$2.9 & 35.3$\pm$2.4 & 6.0$\pm$0.0 & 18.0$\pm$1.2 & 49.3$\pm$4.7 \\
\toprule 
& & & \multicolumn{3}{c}{Coffee D2} & \multicolumn{3}{c}{Three Pc. Assembly D2} & \multicolumn{3}{c}{Hammer Cleanup D1} & \multicolumn{3}{c}{Mug Cleanup D1} \\
\cmidrule(lr){4-6} \cmidrule(lr){7-9} \cmidrule(lr){10-12} \cmidrule(lr){13-15}
Method &Obs & NFE&100 & 200 & 1000 & 100 & 200 & 1000 & 100 & 200 & 1000 & 100 & 200 & 1000 \\
\midrule
\multirow{3}{*}{Ours} & \multirow{3}{*}{RGB} & 1 & 64.7$\pm$5.2 & 80.7$\pm$0.9 & 78.7$\pm$0.9 & 10.7$\pm$2.5 & 35.3$\pm$4.1 & 60.0$\pm$1.6 & 74.7$\pm$1.9 & 75.3$\pm$3.4 & 77.3$\pm$7.5 & 50.0$\pm$1.6 & 64.7$\pm$1.9 & 66.7$\pm$0.9 \\
& & 3 & 66.0$\pm$0.0 & 80.0$\pm$1.6 & 84.0$\pm$1.6 & 10.7$\pm$1.9 & 38.0$\pm$4.3 & 69.3$\pm$3.4 & 72.0$\pm$4.3 & 75.3$\pm$3.4 & 84.0$\pm$2.8 & 50.0$\pm$4.3 & 64.7$\pm$0.9 & 70.0$\pm$1.6 \\
& & 5 & 67.3$\pm$1.9 & 79.3$\pm$1.9 & 82.7$\pm$0.9 & 10.7$\pm$0.9 & 42.0$\pm$3.3 & 70.7$\pm$2.5 & 74.0$\pm$4.3 & 74.7$\pm$3.8 & 82.7$\pm$5.7 & 50.0$\pm$3.3 & 68.0$\pm$1.6 & 70.0$\pm$2.8 \\
\midrule
EquiDiff & RGB & 100 & 60.0$\pm$2.0 & 79.3$\pm$1.3 & 76.0$\pm$2.0 & 15.3$\pm$1.8 & 39.3$\pm$1.8 & 69.3$\pm$3.5 & 65.3$\pm$0.7 & 63.3$\pm$4.4 & 76.7$\pm$0.7 & 49.3$\pm$0.7 & 64.0$\pm$1.2 & 66.7$\pm$0.7 \\
DP-C & RGB & 100 & 44.0$\pm$1.2 & 66.0$\pm$2.3 & 78.7$\pm$0.7 & 4.0$\pm$0.0 & 6.0$\pm$1.2 & 30.0$\pm$1.2 & 52.0$\pm$1.2 & 58.7$\pm$1.3 & 73.3$\pm$2.4 & 42.7$\pm$0.7 & 58.7$\pm$1.3 & 65.3$\pm$2.4 \\
DP-T & RGB & 100 & 47.3$\pm$0.7 & 60.7$\pm$1.8 & 74.7$\pm$2.7 & 0.7$\pm$0.7 & 4.0$\pm$0.0 & 42.7$\pm$1.3 & 48.0$\pm$1.2 & 60.0$\pm$1.2 & 76.0$\pm$1.2 & 30.0$\pm$1.2 & 42.7$\pm$2.9 & 63.3$\pm$0.7 \\
DP3 & PCD & 10 & 34.0$\pm$4.0 & 45.3$\pm$4.1 & 68.7$\pm$2.4 & 0.0$\pm$0.0 & 0.7$\pm$0.7 & 3.3$\pm$0.7 & 54.0$\pm$3.1 & 70.7$\pm$4.1 & 86.7$\pm$0.7 & 21.3$\pm$2.7 & 32.7$\pm$1.8 & 52.7$\pm$4.4 \\
ACT & RGB & 1 & 19.3$\pm$2.4 & 33.3$\pm$2.4 & 64.0$\pm$2.3 & 0.0$\pm$0.0 & 3.3$\pm$0.7 & 24.0$\pm$3.1 & 38.0$\pm$4.2 & 54.0$\pm$1.2 & 70.7$\pm$1.3 & 23.3$\pm$0.7 & 31.3$\pm$1.3 & 56.0$\pm$2.0 \\
\toprule
&& & \multicolumn{3}{c}{Kitchen D1} & \multicolumn{3}{c}{Pick Place D0} & \multicolumn{3}{c}{Nut Assembly D0} & \multicolumn{3}{c}{Coffee Preparation D1} \\
\cmidrule(lr){4-6} \cmidrule(lr){7-9} \cmidrule(lr){10-12} \cmidrule(lr){13-15}
Method &Obs & NFE& 100 & 200 & 1000 & 100 & 200 & 1000 & 100 & 200 & 1000 & 100 & 200 & 1000 \\
\midrule
\multirow{3}{*}{Ours} & \multirow{3}{*}{RGB} & 1 & 65.7$\pm$2.1 & 78.0$\pm$2.8 & 80.7$\pm$4.1 & 37.3$\pm$2.8 & 49.5$\pm$3.1 & 67.3$\pm$0.2 & 59.0$\pm$1.4 & 82.7$\pm$2.6 & 94.3$\pm$0.5 & 75.3$\pm$1.9 & 74.0$\pm$3.3 & 70.0$\pm$2.8 \\
& & 3 & 72.7$\pm$1.9 & 80.7$\pm$0.9 & 80.7$\pm$0.9 & 40.8$\pm$2.7 & 61.7$\pm$3.3 & 85.7$\pm$1.6 & 61.0$\pm$1.6 & 86.3$\pm$0.5 & 96.0$\pm$0.0 & 80.7$\pm$0.9 & 81.3$\pm$2.5 & 88.7$\pm$1.9 \\
& & 5 & 72.7$\pm$1.9 & 81.3$\pm$1.9 & 83.3$\pm$2.5 & 41.2$\pm$0.5 & 65.5$\pm$2.1 & 86.8$\pm$3.1 & 62.3$\pm$1.7 & 87.0$\pm$0.8 & 97.7$\pm$1.7 & 82.7$\pm$1.9 & 82.0$\pm$1.6 & 87.3$\pm$2.5 \\
\midrule
EquiDiff & RGB & 100 & 67.3$\pm$0.7 & 76.7$\pm$3.3 & 81.3$\pm$0.7 & 41.7$\pm$3.2 & 74.2$\pm$3.2 & 92.0$\pm$1.2 & 74.0$\pm$1.2 & 85.0$\pm$1.5 & 93.7$\pm$0.9 & 76.7$\pm$0.7 & 82.7$\pm$0.7 & 85.3$\pm$0.7 \\
DP-C & RGB & 100 & 66.7$\pm$2.4 & 84.7$\pm$0.7 & 86.7$\pm$1.8 & 35.3$\pm$2.2 & 65.0$\pm$2.8 & 82.7$\pm$0.6 & 54.7$\pm$2.3 & 68.0$\pm$2.6 & 83.0$\pm$1.5 & 65.3$\pm$0.7 & 62.0$\pm$4.2 & 58.0$\pm$3.1 \\
DP-T & RGB & 100 & 54.0$\pm$2.3 & 75.3$\pm$0.7 & 81.3$\pm$2.4 & 14.7$\pm$1.5 & 36.5$\pm$1.3 & 50.0$\pm$6.0 & 30.7$\pm$5.0 & 32.3$\pm$5.2 & 45.7$\pm$5.9 & 38.0$\pm$2.0 & 51.3$\pm$1.8 & 76.0$\pm$6.0 \\
DP3 & PCD & 10 & 44.7$\pm$1.8 & 71.3$\pm$2.4 & 91.3$\pm$2.4 & 11.7$\pm$0.9 & 15.0$\pm$1.7 & 34.0$\pm$0.0 & 15.7$\pm$1.3 & 23.7$\pm$3.4 & 57.7$\pm$1.9 & 10.0$\pm$2.3 & 22.0$\pm$5.3 & 63.3$\pm$4.1 \\
ACT & RGB & 1 & 37.3$\pm$3.5 & 60.7$\pm$3.5 & 87.3$\pm$3.5 & 7.2$\pm$0.9 & 17.2$\pm$1.1 & 50.0$\pm$2.9 & 42.3$\pm$2.9 & 63.7$\pm$3.5 & 84.3$\pm$0.9 & 32.0$\pm$2.0 & 46.0$\pm$3.1 & 64.7$\pm$2.4 \\
\bottomrule
\end{tabular}}
\caption{The performance of our EfficientFlow compared with the baselines in MimicGen. We experiment with 100, 200, and 1000 demos in each environment and report the maximum
task success rate among 50 evaluations throughout training. Results averaged over three seeds. $\pm$ indicates standard deviation.}
\label{table1}
\end{table}

\section{Simulation Environments}\label{sec:simulation}

Figure~\ref{image-Mimicgen} presents agent-view observations from the 12 manipulation tasks within the MimicGen~\citep{mandlekar2023mimicgen} simulation environment. As illustrated, these tasks vary significantly in complexity and the number of objects involved. For clarity in our analysis, these tasks can be broadly categorized as follows:

\begin{enumerate}

\item Basic Tasks (Stack, Stack Three): This category comprises a set of box stacking tasks primarily designed to evaluate the fundamental precision of the robot's motion control.
\item Contact-Rich Tasks (Square, Threading, Coffee, Three Piece Assembly, Hammer Cleanup, Mug Cleanup): This group includes tasks that necessitate behaviors with substantial physical contact, such as insertions or drawer articulations. These tasks assess the robot's capability for fine-grained manipulation and its adaptability to uncertainties arising from physical interactions.  
\item Long-Horizon Tasks (Nut Assembly, Kitchen, Pick Place, Coffee Preparation): These tasks require the sequential execution of multiple distinct behaviors, thereby testing the stability of the robot's long-duration movements and its comprehensive ability to perform error recovery when necessary.
\end{enumerate}

In our experiments, both agent-view and eye-in-hand image observations were captured at a resolution of $84\times84$ pixels with 3 color channels (RGB). Point cloud observations consisted of 1024 points, with each point represented by 6 features (XYZ coordinates and RGB color).

\label{Appendix-SimulationEnvironments}
\begin{figure}[htbp]
  \centering   
  \includegraphics[width=0.8\textwidth]{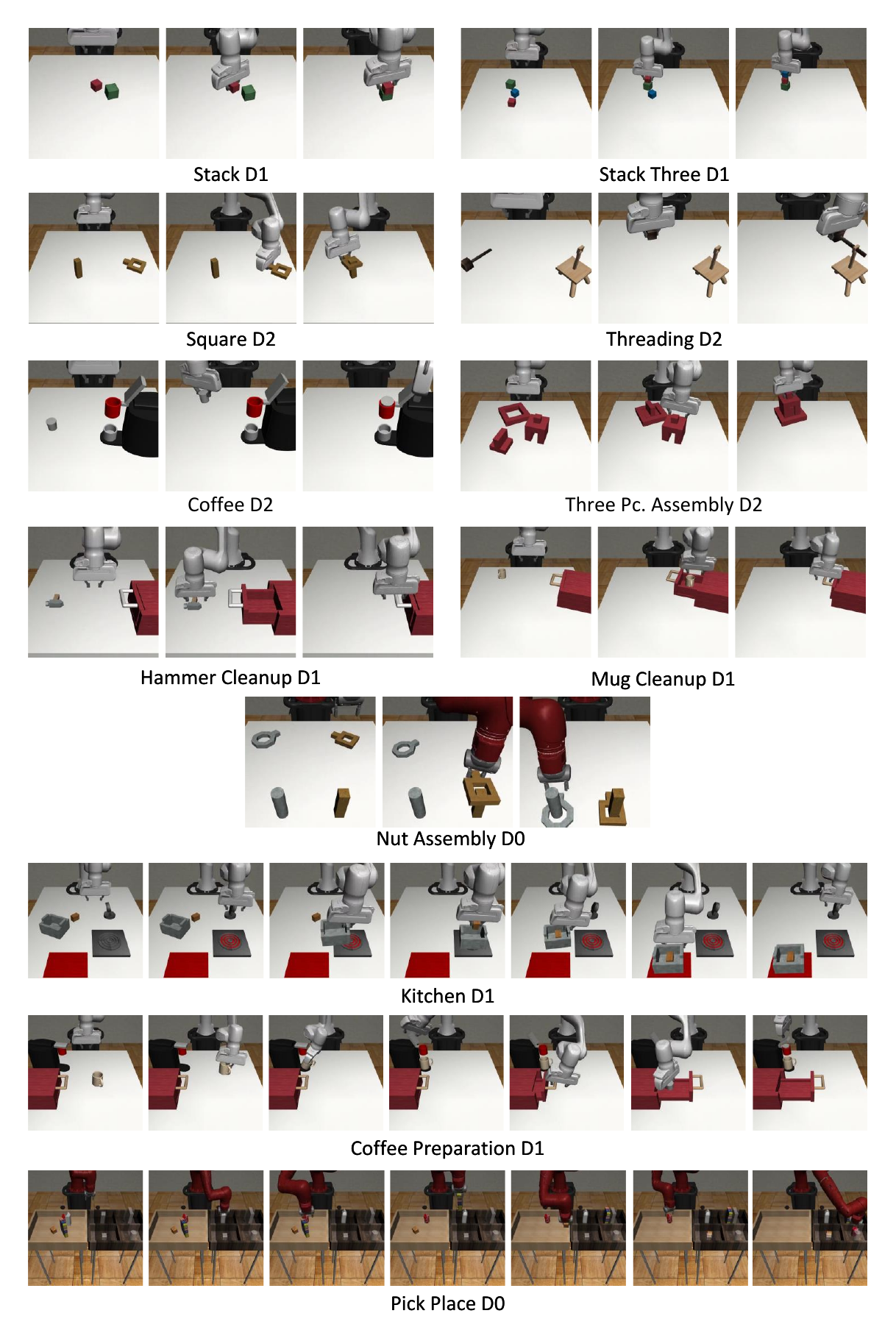}
  \vspace{-10pt}
  \caption{Environment diagrams depicting the MimicGen~\citep{mandlekar2023mimicgen} simulation experiments. The image sequences for each task, presented from left to right, illustrate the progression from the initial state to the final completion of the respective tasks.}
  \label{image-Mimicgen}

\end{figure}
\begin{figure}[htbp]
  \centering    
  \includegraphics[width=0.8\textwidth]{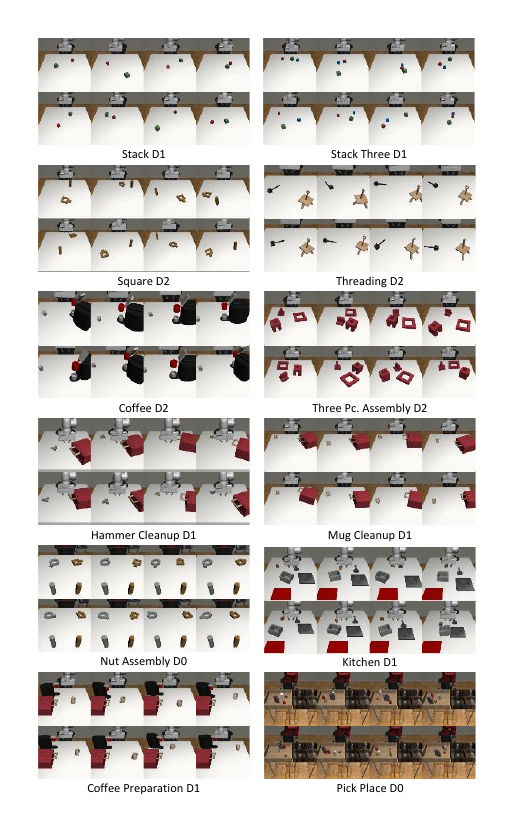}
  \vspace{-30pt}
  
  \caption{The reset distributions for each task in MimicGen~\citep{mandlekar2023mimicgen} simulation experiments.}
  \label{image-MimicgenReset}

\end{figure}

\section{Implementation Details}
\label{Implementation Details}

\subsection{Network Architecture}
For the network architecture, an equivariant ResNet-18 is employed to encode the agent-view images, yielding an output dimensionality of $256 \times 8$. Concurrently, images from the hand camera are processed by a standard non-equivariant ResNet-18, resulting in a 256-dimensional feature vector. These visual features, in conjunction with proprioceptive robot state information, are subsequently fused and compressed via an equivariant layer, producing a combined embedding of $256 \times 8$ dimensions. 

Following this, a time step $t \in [0,1]$ is randomly initialized, along with an initial noise action sequence $x_0$ sampled from a prior distribution. The intermediate action state $x_t$ at time step $t$ is obtained through linear interpolation between $x_0$ and the target action sequence $x_1$. This $x_t$ is then encoded by an equivariant action encoder into a $64 \times 8$ dimensional action embedding. 

The aforementioned embeddings serve as conditioning inputs for a 1D-UNet. This network, featuring hidden layer dimensions of $[512, 1024, 2048]$, predicts a $64 \times 8$ dimensional vector. Finally, this vector is equivariantly decoded to generate the velocity prediction $u_\theta$. The terminal action trajectory is then computed using the Euler method.


\subsection{Training Details}

We train our models with the AdamW~\citep{loshchilov2017decoupled} optimizer with a learning rate of $10^{-4}$ and weight decay of $10^{-6}$ (the learning rate in Coffee Preparation D1, Pick Place D0, and Hammer Cleanup D1 tasks is 0.001). We use a cosine learning rate scheduler with 500 warm-up steps. We conducted training on two types of graphics cards, 4090 and A100. The batch size we used is 80. For different tasks in MimicGen, the training on 4090 requires 23 to 82 hours, respectively. During training, the model receives the two most recent historical observations at each step. A single model output consists of an action sequence spanning 16 time steps, of which the $[1,8]$ steps are executed. The total number of training steps was kept consistent across experiments with varying numbers of demonstrations.
For each different number of demos (100, 200, 1000), we maintain roughly the same number of training steps by training for 50000/n epochs, where n is the number of demos. Evaluations are conducted every 1000/n epochs (50 evaluations in total).

For baselines~\citep{wang2024equivariant, chi2023diffusion, Ze2024DP3}, we adopted the hyperparameter configurations reported in their original publication, except that we use the same action sequence
length (16 for training and 8 for evaluation) in DP3~\citep{Ze2024DP3} as~\citep{wang2024equivariant,chi2023diffusion} and our method. For the ACT~\citep{zhao2023learning}, we follow the hyperparameters provided in the prior work, except that we use a chunk size of 10, a KL weight of 10, a batch size of 64 with a learning rate of $5\times10^{-5}$, and no temporal aggregation, following the tuning tips provided by the authors.

\subsection{Evaluation Strategy}
During the evaluation phase, the model similarly processes the two most recent observations. To enhance temporal consistency, five independent initial noise action sequences are randomly sampled and processed in parallel by the network. These sequences undergo an iterative inference process for 1, 3, or 5 steps (NFE=1, 3, or 5), resulting in five distinct candidate action trajectories.

To ensure smooth action transitions, the Euclidean distance is computed between each newly generated candidate trajectory and the terminal segment of the previously predicted trajectory. Specifically, this involves comparing the last 7 steps of the previous trajectory with the initial 1-8 steps of the current candidate trajectory. The candidate trajectory exhibiting the smallest Euclidean distance is then selected for execution. Furthermore, to encourage exploration, approximately every 10 predictions, a trajectory is chosen randomly from the candidates instead of always selecting the one with the smoothest transition.

\section{Additional Analysis}
\subsection{Hyperparameter Sensitivity Analysis for \texorpdfstring{$\lambda$}{lambda}}

The FABO module is critical for the model's performance. Given that its influence is modulated by the hyperparameter $\lambda$, we performed a sensitivity analysis on the formulation of $\lambda$ for the Mug Cleanup D1 task, which is particularly sensitive to this component. The formulation adopted in this work is $\lambda = (1-t)^2$. We tested several alternative formulations to validate this choice in Table~\ref{tab:lambda_sensitivity}. 

The result reveals two key insights. First, the time-varying characteristic is essential, as replacing it with a constant schedule degrades the success rate from 50\% to 42.0\%. Second, the model exhibits significant robustness to the scale of this formulation: multiplying the original schedule by factors of 0.5, 1, or 2 yields comparable results. This insensitivity suggests that precise calibration of the magnitude is unnecessary, effectively easing the hyperparameter tuning overhead.

\begin{table}[htbp]
\centering
\begin{tabular}{lc}
\toprule
\textbf{Hyperparameter Setting for $\lambda$} & \textbf{Mean Success Rate (\%)} \\
\midrule
\multicolumn{2}{l}{\textit{Time-Varying Formulations}} \\
\quad $0.5(1-t)^2$       & $48.0 \pm 1.6$ \\
\quad $(1-t)^2$         & $50.0 \pm 1.6$ \\
\quad $2(1-t)^2$        & $51.3 \pm 3.8$ \\
\midrule
\multicolumn{2}{l}{\textit{Constant Formulations}} \\
\quad $0.5$             & $42.0 \pm 2.8$ \\
\bottomrule
\end{tabular}
\caption{Sensitivity analysis of the hyperparameter $\lambda$ on the Mug Cleanup D1 task. The maximum task success rate among 50 evaluations throughout training and the standard deviations are reported.}
\label{tab:lambda_sensitivity}
\end{table}

\subsection{Analysis of Trajectory Quality}
To quantify trajectory smoothness, we measured the rate of change in velocity at 500 sampled timesteps on the Stack D1 task. EfficientFlow exhibits a mean velocity change of 0.103 (std: 0.088), representing a significant reduction of 24.3\% compared to the NoAcc baseline (mean: 0.136, std: 0.133).

\section{Multi-Modal Extensions and Generalization}
\subsection{Multi-Modal Performance in MimicGen}
To further investigate the adaptability and potential of our core architecture, we extend EfficientFlow to incorporate 3D geometric information through a voxel-based representation.

We implemented a voxel-based variant of EfficientFlow and compared it against our original RGB-based model from the main paper. For context, we also include results from a strong point-cloud-based method, Flowpolicy~\citep{zhang2025flowpolicy}, and its baseline, DP3~\citep{Ze2024DP3}. The evaluation was conducted in the MimicGen environment on five tasks (Stack D1, Threading D2, Square D2, Stack Three D1, and Three Pc. Asse. D2), with each model trained on 100 demonstrations. We report the mean of maximum success rates over three random seeds in Table~\ref{table2}.

The Voxel-based EfficientFlow achieves superior performance by leveraging richer spatial perception and explicit 3D geometry. This demonstrates that our strategy effectively generalizes across different input modalities. However, the acquisition overhead of 3D data—ranging from sensor cost to real-time processing—poses a barrier to real-world deployment. Consequently, while Voxels offer peak performance, our RGB variant remains a vital solution for scenarios where simplicity and hardware accessibility are prioritized.

\newcolumntype{C}[1]{>{\centering\arraybackslash}p{#1}}
\begin{table}[ht!]
\centering
\begin{tabular}{l *{6}{C{1.9cm}}}
\toprule
\textbf{Method} & \textbf{Stack} & \textbf{Threading} & \textbf{Square} & \textbf{Stack 3} & \textbf{3Pc. Asm.} & \textbf{Average} \\
\midrule
DP3             & 69.3$\pm$3.7      & 12.0$\pm$3.1          & 6.7$\pm$0.7        & 7.3$\pm$0.7             & 0.0$\pm$0.0             & 19.1 \\
FlowPolicy      & 72.0$\pm$7.1      & 13.3$\pm$1.9          & 6.0$\pm$1.6        & 10.0$\pm$1.6            & 0.0$\pm$0.0             & 20.3 \\

\midrule
Ours(RGB)       & 94.0$\pm$1.6      & 31.3$\pm$1.9          & 20.7$\pm$1.9       & 48.0$\pm$0.0            & 10.7$\pm$2.5            & 41.0 \\
Ours(Voxel)     & 93.3$\pm$0.9      & 41.3$\pm$0.9          & 33.3$\pm$0.9       & 67.3$\pm$6.2            & 20.0$\pm$3.3            & 51.0 \\
\bottomrule
\end{tabular}
\caption{Multi-modal performance comparison in the MimicGen environment. We report the mean success rates (\%) and standard deviations over three random seeds. Our EfficientFlow framework, in both RGB and Voxel configurations, shows superior performance.}
\label{table2}
\end{table}

\subsection{Robomimic Experiment}
To further validate the generalization capabilities of our method beyond the training environment, we conducted experiments on the Robomimic~\citep{robomimic2021} benchmark. We trained EfficientFlow and Diffusion Policy~\citep{chi2023diffusion} using only 20 expert demonstrations for four proficient-human (ph) single-arm tasks. Other hyperparameters mirror those used in our MimicGen experiment.

Due to the limited randomness in the initial state distributions of these tasks, the data efficiency gains stemming from equivariance are less pronounced compared to the MimicGen experiments.
Nevertheless, EfficientFlow consistently outperforms the baseline across the majority of tasks, securing a superior average success rate.
\begin{table}[htb!]
\centering
\begin{tabular}{l *{5}{C{2.cm}}}
\toprule
\textbf{Method} & {\textbf{Tool hang}} & {\textbf{Can}} & {\textbf{Lift}} & {\textbf{Square}} & \textbf{Average} \\
\midrule
DP-C & 15.3 $\pm$ 2.5 & 67.3 $\pm$ 5.0 & 100.0 $\pm$ 0.0 & 42.7 $\pm$ 4.7  &56.3 \\
Ours & 16.7 $\pm$ 5.7 & 90.7 $\pm$ 0.9 & 100.0 $\pm$ 0.0 & 44.0 $\pm$ 5.9 &62.9 \\
\bottomrule
\end{tabular}
\caption{Performance comparison on Robomimic tasks. We report the maximum
task success rate among 50 evaluations throughout training and standard deviations over three random seeds.}
\label{tab:robomimic_tasks}
\end{table}

\clearpage
\bibliographystyle{iclr2025_conference}
\bibliography{main}

\end{document}